\documentclass[twocolumn,conference,letterpaper]{IEEEtran}
\usepackage[left=0.75in,right=0.75in,top=1in,bottom=0.75in]{geometry}

\usepackage{amssymb,amsmath,latexsym,amsfonts,amsthm,mathtools}
\usepackage{graphicx}
\usepackage{float,subcaption}
\usepackage{textcomp}
\usepackage{xcolor}
\usepackage{booktabs}
\usepackage{cite}

\usepackage[hidelinks]{hyperref}
\hypersetup{colorlinks, breaklinks, citecolor=blue, linkcolor=blue, urlcolor=blue}
\usepackage[nameinlink]{cleveref}
\Crefformat{figure}{#2Fig.~#1#3}
\Crefmultiformat{figure}{Figs.~#2#1#3}{ and~#2#1#3}{, #2#1#3}{ and~#2#1#3}

\usepackage{tikz}
\usetikzlibrary{automata, shapes, arrows, calc, arrows.meta, fit, positioning}

\theoremstyle{plain}
\newtheorem{theorem}{Theorem}
\newtheorem{lemma}{Lemma}
\newtheorem{corollary}{Corollary}
\newtheorem{proposition}{Proposition}
\newtheorem*{problem*}{Problem}
\theoremstyle{remark}

\newtheorem{assumption}{Assumption}

\newtheorem{definition}{Definition}
\theoremstyle{definition}

\usepackage{mathrsfs}

\usepackage{newtxmath}

\DeclareMathOperator{\rank}{rank}
\DeclareMathOperator{\dist}{dist}
\DeclareMathOperator{\blkdiag}{blkdiag}
\DeclareMathOperator*{\argmax}{arg\,max}
\DeclareMathOperator*{\argmin}{arg\,min}
\newcommand{\R}{\mathbb{R}}
\newcommand{\W}{\mathcal{W}}
\newcommand{\Z}{\mathcal{Z}}
\newcommand{\T}{\mathcal{T}}
\newcommand{\Aset}{\mathcal{A}}
\newcommand{\B}{\mathbb{B}}
\newcommand{\M}{\mathbf{M}}
\newcommand{\A}{\mathbf{A}}
\newcommand{\e}{\mathbf{e}}
\newcommand{\w}{\mathbf{w}}
\newcommand{\z}{\mathbf{z}}
\newcommand{\one}{\mathbf{1}}
\newcommand{\zero}{\mathbf{0}}

\newcommand{\intset}{\operatorname{int}}
\newcommand{\bd}{\partial}

\newcommand{\SO}{\mathrm{SO}}

\definecolor{cNavy}{RGB}{20,45,70}
\definecolor{cTeal}{RGB}{42,132,130}
\definecolor{cGold}{RGB}{190,135,55}
\definecolor{cGray}{RGB}{235,238,240}

\IEEEoverridecommandlockouts

\begin{document}

\title{Contact-Persistent Full Actuation for Aerial Physical Interaction}

\author{Abhimanyu Khadga, Abhinav Sinha,~\IEEEmembership{Senior Member,~IEEE}, and Shashi Ranjan Kumar,~\IEEEmembership{Senior Member,~IEEE}
\thanks{A. Khadga and A. Sinha are with the Guidance, Autonomy, Learning, and Control for Intelligent Systems (GALACxIS) Lab, Department of Aerospace Engineering and Engineering Mechanics, University of Cincinnati, OH 45221, USA (e-mails: khadgaau@mail.uc.edu, abhinav.sinha@uc.edu). S. R. Kumar is with the Intelligent Systems and Control (ISaC) Lab, Department of Aerospace Engineering, Indian Institute of Technology Bombay, Powai, Mumbai 400076, India (email: srk@aero.iitb.ac.in).}
}

\maketitle
\thispagestyle{empty}

\begin{abstract}
Fully actuated unmanned aerial vehicles (UAVs) are usually certified through rank conditions on a control-allocation matrix or through free-flight tracking performance. For aerial physical interaction, this certification may be incomplete. During sustained contact, part of the available wrench is consumed by the interaction task, and only the residual wrench remains available for stabilization, disturbance rejection, and maneuvering. This paper introduces a control-theoretic framework for \emph{contact-persistent full actuation}. A rigid-body model on $\R^{3}\times\SO\left(3\right)$ is combined with a morphology-dependent wrench map that captures fixed-tilt, variable-tilt, coaxial, and overactuated multirotor architectures. We define feasible wrench sets under actuator limits, residual wrench sets under task loading, and residual authority margins that strengthen the usual rank-based notion of full actuation. The main result shows that contact-persistent full actuation is equivalent to interiority of the task wrench in the constrained feasible wrench polytope, and that the residual authority radius is exactly the distance to the polytope boundary. We further introduce a signed residual-margin certificate for infeasible and boundary cases, a slack-maximizing allocation certificate, and a robust implementability condition that can be used as a margin-aware safety filter. Numerical evaluation on an abstract tilted hexarotor shows that full row rank alone does not imply feasible contact operation. Intermediate tilt angles preserve residual authority during pushing, whereas small or excessive tilts fail because of lateral-force deficiency or hover-margin loss.
\end{abstract}

\begin{IEEEkeywords}
Fully actuated UAVs, aerial physical interaction, control allocation, residual wrench authority, actuator constraints, contact-rich robotics.
\end{IEEEkeywords}

\section{Introduction}

Multirotor UAVs have become standard platforms for agile robotics because they are mechanically simple, hover-capable, and easy to deploy. Conventional coplanar multirotors, however, are underactuated. The propellers produce nearly parallel thrust directions, and lateral acceleration is obtained by reorienting the body. This coupling is acceptable for free-flight trajectory tracking \cite{rashad2020review,hamandi2021taxonomy}, but it is restrictive when the vehicle must maintain a desired attitude while exerting forces on the environment. Fully actuated and overactuated multirotors address this limitation by using fixed tilted rotors, active tilting rotors, noncoplanar propeller arrangements, coaxial thrust-vectoring designs, or aerodynamic surfaces to generate a six-dimensional wrench \cite{rajappa2015modeling,ryll2015overactuatedquadrotor,brescianini2016design,brescianini2018omni,rashad2017design,rashad2020review,hamandi2021taxonomy,kotarski2021performance}.

The literature on fully actuated aerial vehicles has developed along three closely related lines. The first line is architectural wherein tilted-propeller hexarotors, overactuated quadrotors, omnidirectional octorotors, and morphing/variable-tilt vehicles have been designed to decouple force generation from attitude and to enlarge the attainable wrench set \cite{rajappa2015modeling,ryll2015overactuatedquadrotor,brescianini2016design,brescianini2018omni,allenspach2020design,ryll2022fasthex}. The second line is interaction-oriented wherein fully actuated aerial robots have been used as flying end-effectors, contact-inspection systems, omnidirectional manipulation platforms, and heavy-object pushing vehicles \cite{park2018odar,ryll2019flyingend,bodie2019omnidirectional,bodie2020active,hwang2024pushing,hui2024passive,veenstra2026sixd,lin2025float}. The third line concerns allocation and implementation wherein modern schemes account for task priority, actuator limits, singularities, overactuation, power dynamics, and saturation avoidance \cite{allenspach2020design,hwang2024pushing,cuniato2024allocation}.

Despite these advances, most certification arguments still focus on rank, wrench-generation ability, tracking performance, or force-control performance for a specific experiment. These criteria are necessary but not sufficient for sustained contact-rich operation. During physical interaction, part of the feasible wrench set is consumed by hover, pushing, bracing, or end-effector loading, and the remaining portion must still support stabilization, disturbance rejection, and maneuvering. Thus, the relevant question is not solely whether a desired contact wrench is feasible, but whether it is feasible while preserving a nonzero stabilizing wrench margin.

This paper is built around the following observation. The common condition of having the rank of the allocation matrix equal to six certifies structural wrench generation, but it does not certify contact-rich operation under actuator bounds. A fully actuated UAV may have rank six and still fail to push, brace, or reject disturbances if the required task wrench lies near the boundary of the feasible wrench set. For example, a horizontal pushing task consumes lateral force authority while gravity compensation consumes vertical authority. If the corresponding actuator allocation is close to saturation, the vehicle may have no meaningful wrench margin left for stabilizing perturbations. Hence, contact-rich actuation should be evaluated by the dimension of the wrench map, as well as by the residual wrench set left after task loading.

The contribution of this paper is an actuation-theoretic certification framework for contact-rich fully actuated UAVs. Unlike rank-based tests, the proposed framework certifies whether a contact task leaves a nonzero residual wrench neighborhood for stabilization and disturbance rejection under actuator limits. First, we formulate morphology-dependent feasible wrench polytopes for a broad class of tilted, coaxial, variable-tilt, and overactuated multirotors. Second, we define contact-persistent full actuation as a task-relative interiority property of the constrained feasible wrench set. Third, we derive exact and computable residual-authority certificates, including signed margins, task-set margins, and allocation-slack bounds. Fourth, we derive robust implementability and margin-aware filtering conditions that connect residual authority to closed-loop wrench feasibility. Finally, we show through a tilted-hexarotor study that rank-six allocation can still fail under sustained pushing, and that intermediate or task-adaptive tilt choices maximize residual stabilizing authority.

The paper is intentionally hardware-independent. It provides a way to ask, before building a prototype, whether a candidate fully actuated UAV morphology is suitable for sustained contact-rich operation. The resulting residual authority metric is also compatible with later work on morphology-aware allocation, saturation-aware control, and safety filters.

\section{Modeling and Feasible Wrench Sets}
\label{sec:model}
Consider a UAV modeled as a rigid body evolving on $\R^{3}\times\SO\left(3\right)$. Let $\mathbf{p,v}\in\R^{3}$ be the inertial position and velocity, $\mathbf{R}\in\SO\left(3\right)$ be the rotation matrix from body to inertial coordinates, and $\boldsymbol{\omega}\in\R^{3}$ be the body angular velocity. The dynamics of the UAV can be written as
\begin{subequations}
\begin{align}
    \dot{\mathbf{p}} &= \mathbf{v},
    \label{eq:p_dot}\\
    m\dot{\mathbf{v}} &= -mg\e_{3}+\mathbf{R}\mathbf{f}+\mathbf{d}_{f},
    \label{eq:v_dot}\\
    \dot{\mathbf{R}} &= \mathbf{R}S\left(\boldsymbol{\omega}\right),
    \label{eq:R_dot}\\
    \mathbf{J}\dot{\boldsymbol{\omega}} &= -\boldsymbol{\omega}\times\mathbf{J}\boldsymbol{\omega}+\boldsymbol{\tau}+\mathbf{d}_{\tau},
    \label{eq:omega_dot}
\end{align}
\end{subequations}
where $m>0$ is the mass, $\mathbf{J}=\mathbf{J}^{\top}>0$ is the inertia matrix, $g>0$ is gravitational acceleration, $\e_{3}=\left[0\;0\;1\right]^{\top}$, $\mathbf{f}$ and  $\mathbf{\tau}$ denotes the body-frame control force and torque, and $\mathbf{d}_{f}$ and $\mathbf{d}_{\tau}$ denote unmodeled external effects. The map $S:\R^{3}\rightarrow\R^{3\times3}$ is the skew-symmetric operator satisfying $S\left(\mathbf{x}\right)\mathbf{y}=\mathbf{x}\times\mathbf{y}$ for all $\mathbf{x},\mathbf{y}\in\R^{3}$. The body wrench is
\begin{align}
\w=\begin{bmatrix}\mathbf{f}^{\top}&\boldsymbol{\tau}^{\top}\end{bmatrix}^{\top}\in\R^{6}.
\end{align}
We assume that the vehicle has $N$ thrust-generating modules. The $i$th module is located at $\mathbf{r}_{i}\in\R^{3}$ in body coordinates and produces thrust magnitude $\lambda_{i}\geq0$ along the unit direction $\mathbf{b}_{i}\left(\boldsymbol{\alpha}_{i}\right)\in\R^{3}$. Here $\boldsymbol{\alpha}_{i}$ denotes architecture-dependent variables such as fixed tilt angles, servo tilt angles, or local rotor-frame orientation variables. The force and torque generated by module $i$ are modeled as

\begin{align}
    \mathbf{f}_{i} &= \lambda_{i}\mathbf{b}_{i}\left(\boldsymbol{\alpha}_{i}\right),\\
    \boldsymbol{\tau}_{i} &= \mathbf{r}_{i}\times \lambda_{i}\mathbf{b}_{i}\left(\boldsymbol{\alpha}_{i}\right)+\sigma_{i}c_{i}\lambda_{i}\mathbf{b}_{i}\left(\boldsymbol{\alpha}_{i}\right),
\end{align}
where $c_{i}\geq0$ is a drag-to-thrust coefficient and $\sigma_{i}\in\{-1,1\}$ is the spin direction. Therefore,
\begin{align}
    \w=\sum_{i=1}^{N}
    \begin{bmatrix}
        \mathbf{b}_{i}\left(\boldsymbol{\alpha}_{i}\right)\\
        S\left(\mathbf{r}_{i}\right)\mathbf{b}_{i}\left(\boldsymbol{\alpha}_{i}\right)+\sigma_{i}c_{i}\mathbf{b}_{i}\left(\boldsymbol{\alpha}_{i}\right)
    \end{bmatrix}\lambda_{i}.
    \label{eq:wrench_sum_short}
\end{align}
Let $\boldsymbol{\lambda}=\left[\lambda_{1}\;\cdots\;\lambda_{N}\right]^{\top}$ and let $\boldsymbol{\alpha}$ collect all morphology variables. Then, \eqref{eq:wrench_sum_short} is written compactly as $\w=\M\left(\boldsymbol{\alpha}\right)\boldsymbol{\lambda}$, where $\M\left(\boldsymbol{\alpha}\right)\in\R^{6\times N}$ is the allocation matrix. The actuator vector is constrained by $\boldsymbol{\lambda}\in\Lambda:=\left\{\boldsymbol{\lambda}\in\R^{N}:\boldsymbol{\lambda}_{\min}\leq\boldsymbol{\lambda}\leq\boldsymbol{\lambda}_{\max}\right\}$, where inequalities are componentwise. The morphology satisfies $\boldsymbol{\alpha}\in\Aset$, where $\Aset$ is compact. For a fixed morphology, the feasible wrench set is
\begin{align}
    \W\left(\boldsymbol{\alpha}\right):=\left\{\w\in\R^{6}:\w=\M\left(\boldsymbol{\alpha}\right)\boldsymbol{\lambda},\;\boldsymbol{\lambda}\in\Lambda\right\}.
    \label{eq:feasible_wrench}
\end{align}

\begin{definition}[Constrained Full Actuation]\label{def:constrained_full_actuation}
For a given operating wrench $\w_{0}$, the UAV is constrained fully actuated at $\left(\boldsymbol{\alpha},\w_{0}\right)$ if $\w_{0}\in\intset\left(\W\left(\boldsymbol{\alpha}\right)\right)$.
\end{definition}
At this stage, it is important to distinguish \Cref{def:constrained_full_actuation} from the condition $\rank\left(\M\right)=6$. Rank is a structural property of $\M\left(\boldsymbol{\alpha}\right)$, whereas constrained full actuation depends on actuator bounds and on the location of the operating wrench inside $\W\left(\boldsymbol{\alpha}\right)$.

\section{Contact-Persistent Full Actuation}
\label{sec:theory}
Let $\w_{\mathrm{task}}\in\R^{6}$ be the wrench committed to a contact task, such as pushing or bracing. The residual wrench set is
\begin{align}
    \W_{\mathrm{res}}\left(\boldsymbol{\alpha},\w_{\mathrm{task}}\right):=\left\{\Delta\w\in\R^{6}:\w_{\mathrm{task}}+\Delta\w\in\W\left(\boldsymbol{\alpha}\right)\right\}.
    \label{eq:residual_set}
\end{align}
The set $\W_{\mathrm{res}}$ contains the stabilizing and disturbance-rejection wrenches that remain after task loading.
\begin{definition}[Contact-persistent Full Actuation]
Given $\epsilon>0$ and $\w_{\mathrm{task}}$, the UAV is contact-persistently fully actuated at $\boldsymbol{\alpha}$ with margin $\epsilon$ if $\B_{\epsilon}\left(\zero_{6}\right)\subseteq\W_{\mathrm{res}}\left(\boldsymbol{\alpha},\w_{\mathrm{task}}\right)$, where $\B_{\epsilon}\left(\zero_{6}\right)=\{\zeta\in\R^{6}:\left\|\zeta\right\|\leq\epsilon\}$.
\end{definition}
\begin{assumption}
\label{ass:rank}
For the fixed morphology under consideration, $\rank\left(\M\left(\boldsymbol{\alpha}\right)\right)=6$, and $\lambda_{i,\min}<\lambda_{i,\max}$ for all $i\in\{1,\ldots,N\}$.
\end{assumption}

\begin{lemma}[Feasible Wrench Geometry]
\label{lem:geometry}
Under \Cref{ass:rank}, $\W\left(\boldsymbol{\alpha}\right)$ is a compact, convex, full-dimensional polytope in $\R^{6}$.
\end{lemma}

\begin{proof}
The set $\Lambda$ is a compact convex box. Its image under the linear map $\boldsymbol{\lambda}\mapsto\M\left(\boldsymbol{\alpha}\right)\boldsymbol{\lambda}$ is compact and convex. Since $\M\left(\boldsymbol{\alpha}\right)$ has full row rank, the image contains a nonempty open set in $\R^{6}$. Since a linear image of a polytope is a polytope, the claim follows.
\end{proof}
Here and throughout, $\bd\W\left(\boldsymbol{\alpha}\right)$ denotes the topological boundary of $\W\left(\boldsymbol{\alpha}\right)$. Since $\W\left(\boldsymbol{\alpha}\right)$ is compact under \Cref{ass:rank}, this is equivalently $\bd\W\left(\boldsymbol{\alpha}\right)=\W\left(\boldsymbol{\alpha}\right)\setminus\intset\left(\W\left(\boldsymbol{\alpha}\right)\right)$.
\begin{proposition}[Insufficiency of Rank]
\label{prop:rank_not_enough}
\Cref{ass:rank} does not imply that every desired wrench $\w_{d}\in\R^{6}$ is feasible. Moreover, if $\w_{d}\in\bd\W\left(\boldsymbol{\alpha}\right)$, then no $\epsilon>0$ satisfies $\B_{\epsilon}\left(\w_{d}\right)\subseteq\W\left(\boldsymbol{\alpha}\right)$.
\end{proposition}
\begin{proof}
Since $\Lambda$ is compact, $\W\left(\boldsymbol{\alpha}\right)$ is compact and cannot equal $\R^{6}$. Hence some wrenches are infeasible even though the allocation map is full row rank. If $\w_{d}$ lies on the boundary of $\W\left(\boldsymbol{\alpha}\right)$, every open ball centered at $\w_{d}$ contains points outside $\W\left(\boldsymbol{\alpha}\right)$.
\end{proof}

\begin{theorem}[Residual Authority and Interiority]
\label{thm:main}
Let $\w_{\mathrm{task}}\in\W\left(\boldsymbol{\alpha}\right)$. There exists $\epsilon>0$ such that $\B_{\epsilon}\left(\zero_{6}\right)\subseteq\W_{\mathrm{res}}\left(\boldsymbol{\alpha},\w_{\mathrm{task}}\right)$ if and only if $\w_{\mathrm{task}}\in\intset\left(\W\left(\boldsymbol{\alpha}\right)\right)$. Moreover, the largest residual authority radius is
\begin{align}
\epsilon^{\star}\left(\boldsymbol{\alpha},\w_{\mathrm{task}}\right)=\dist\left(\w_{\mathrm{task}},\bd\W\left(\boldsymbol{\alpha}\right)\right).
    \label{eq:exact_margin}
\end{align}
\end{theorem}
\begin{proof}
From \eqref{eq:residual_set}, $\B_{\epsilon}\left(\zero_{6}\right)\subseteq\W_{\mathrm{res}}\left(\boldsymbol{\alpha},\w_{\mathrm{task}}\right)$ if and only if $\w_{\mathrm{task}}+\B_{\epsilon}\left(\zero_{6}\right)\subseteq\W\left(\boldsymbol{\alpha}\right)$. Equivalently, $\B_{\epsilon}\left(\w_{\mathrm{task}}\right)\subseteq\W\left(\boldsymbol{\alpha}\right)$. Such an $\epsilon>0$ exists exactly when $\w_{\mathrm{task}}$ is an interior point, and the largest such $\epsilon$ is the distance from $\w_{\mathrm{task}}$ to the boundary.
\end{proof}

If $\W\left(\boldsymbol{\alpha}\right)$ has the half-space representation
\begin{align}
    \W\left(\boldsymbol{\alpha}\right)=\left\{\w\in\R^{6}:\mathbf{a}_{j}^{\top}\w\leq b_{j},\;j=1,\ldots,q\right\},
    \label{eq:h_rep}
\end{align}
then \eqref{eq:exact_margin} becomes
\begin{align}
    \epsilon^{\star}\left(\boldsymbol{\alpha},\w_{\mathrm{task}}\right)=\min_{j\in\{1,\ldots,q\}}\dfrac{b_{j}-\mathbf{a}_{j}^{\top}\w_{\mathrm{task}}}{\left\|\mathbf{a}_{j}\right\|}.
    \label{eq:h_margin}
\end{align}
Thus, contact-rich actuation can be certified by a distance-to-boundary computation. For later use, it is useful to define a signed version of this margin. In particular, if \eqref{eq:h_rep} is available, define
\begin{align}
    \varphi_{\W}\left(\boldsymbol{\alpha},\w\right)
    :=
    \min_{j\in\{1,\ldots,q\}}
    \dfrac{b_{j}-\mathbf{a}_{j}^{\top}\w}{\left\|\mathbf{a}_{j}\right\|}.
    \label{eq:signed_margin_w}
\end{align}
Then $\varphi_{\W}\left(\boldsymbol{\alpha},\w\right)>0$ if and only if $\w\in\intset\left(\W\left(\boldsymbol{\alpha}\right)\right)$, while $\varphi_{\W}\left(\boldsymbol{\alpha},\w\right)<0$ means that at least one actuator-induced half-space constraint is violated. Hence, unlike the unsigned distance $\dist\left(\w,\bd\W\left(\boldsymbol{\alpha}\right)\right)$, the signed margin does not assign a positive value to infeasible exterior points.

\begin{corollary}[Signed-Margin Robust Ball]
\label{cor:signed_margin_ball}
Assume that $\W\left(\boldsymbol{\alpha}\right)$ admits the half-space representation \eqref{eq:h_rep}. For any $\w_{0}\in\R^{6}$ and any $\epsilon\geq0$,
\begin{align}
\B_{\epsilon}\left(\w_{0}\right)\subseteq\W\left(\boldsymbol{\alpha}\right)
~~\Longleftarrow~~
\varphi_{\W}\left(\boldsymbol{\alpha},\w_{0}\right)\geq\epsilon.
\label{eq:signed_ball_condition}
\end{align}
Moreover, if $\w_{0}\in\W\left(\boldsymbol{\alpha}\right)$, then the largest admissible Euclidean ball centered at $\w_{0}$ has radius $\max\left\{0,\varphi_{\W}\left(\boldsymbol{\alpha},\w_{0}\right)\right\}$.
\end{corollary}
\begin{proof}
For any point $\w_{0}+\Delta\w$ with $\left\|\Delta\w\right\|\leq\epsilon$, each half-space constraint satisfies $\mathbf{a}_{j}^{\top}\left(\w_{0}+\Delta\w\right)\leq \mathbf{a}_{j}^{\top}\w_{0}+\left\|\mathbf{a}_{j}\right\|\epsilon$. If \eqref{eq:signed_ball_condition} holds, then $\mathbf{a}_{j}^{\top}\w_{0}+\left\|\mathbf{a}_{j}\right\|\epsilon\leq b_{j}$ for all $j$, so the ball is contained in $\W\left(\boldsymbol{\alpha}\right)$. The maximal-radius statement follows by minimizing the normalized half-space slack.
\end{proof}
The exact boundary of $\W\left(\boldsymbol{\alpha}\right)$ may be expensive to construct for large overactuated systems. The next result gives a computable lower bound directly from actuator slack.
\begin{theorem}[Interior Allocation Certificate]
\label{thm:certificate}
Let $\mathbf{R}_{M}\in\R^{N\times6}$ satisfy $\M\left(\boldsymbol{\alpha}\right)\mathbf{R}_{M}=\mathbf{I}_{6}$. Suppose there exists $\boldsymbol{\lambda}_{0}\in\intset\left(\Lambda\right)$ such that $\M\left(\boldsymbol{\alpha}\right)\boldsymbol{\lambda}_{0}=\w_{\mathrm{task}}$. Define
\begin{align}
    \delta_{\infty}\left(\boldsymbol{\lambda}_{0}\right)=\min_{i}\left\{\lambda_{0,i}-\lambda_{i,\min},\;\lambda_{i,\max}-\lambda_{0,i}\right\}.
\end{align}
Then the UAV is contact-persistently fully actuated with a margin of at least
\begin{align}
\epsilon_{\mathrm{cert}}=\dfrac{\delta_{\infty}\left(\boldsymbol{\lambda}_{0}\right)}{\left\|\mathbf{R}_{M}\right\|_{2\rightarrow\infty}},
    \label{eq:certificate_margin}
\end{align}
where $\left\|\mathbf{R}_{M}\right\|_{2\rightarrow\infty}:=\sup_{\zeta\neq\zero}\left\|\mathbf{R}_{M}\zeta\right\|_{\infty}/\left\|\zeta\right\|$.
\end{theorem}
\begin{proof}
For any $\Delta\w$ satisfying $\left\|\Delta\w\right\|\leq\epsilon_{\mathrm{cert}}$, set $\Delta\boldsymbol{\lambda}=\mathbf{R}_{M}\Delta\w$. Then $\M\left(\boldsymbol{\alpha}\right)\left(\boldsymbol{\lambda}_{0}+\Delta\boldsymbol{\lambda}\right)=\w_{\mathrm{task}}+\Delta\w$. Moreover, $\left\|\Delta\boldsymbol{\lambda}\right\|_{\infty}\leq\delta_{\infty}\left(\boldsymbol{\lambda}_{0}\right)$, so $\boldsymbol{\lambda}_{0}+\Delta\boldsymbol{\lambda}\in\Lambda$. Hence $\w_{\mathrm{task}}+\Delta\w\in\W\left(\boldsymbol{\alpha}\right)$ for all such $\Delta\w$.
\end{proof}

\begin{corollary}[Slack-Maximizing Allocation Certificate]
\label{cor:slack_max}
For a fixed morphology and a desired task wrench $\w_{\mathrm{task}}$, consider
\begin{align}
\left(\boldsymbol{\lambda}^{\star},\delta^{\star}\right)
    \in
    \argmax_{\boldsymbol{\lambda},\delta}\;&\delta
    \label{eq:slack_lp}\\
    \mathrm{s.t.}\;&\M\left(\boldsymbol{\alpha}\right)\boldsymbol{\lambda}=\w_{\mathrm{task}},\nonumber\\
    &\boldsymbol{\lambda}_{\min}+\delta\one\leq\boldsymbol{\lambda}
    \leq
    \boldsymbol{\lambda}_{\max}-\delta\one,\nonumber\\
    &\delta\geq0.\nonumber
\end{align}
The problem is a linear program. If $\delta^{\star}>0$, then $\w_{\mathrm{task}}\in\intset\left(\W\left(\boldsymbol{\alpha}\right)\right)$ and the UAV is contact-persistently fully actuated. In addition, for any right inverse $\mathbf{R}_{M}$ of $\M\left(\boldsymbol{\alpha}\right)$, the allocation $\boldsymbol{\lambda}^{\star}$ certifies the residual margin
\begin{align}
\epsilon_{\mathrm{slack}}
    =
    \dfrac{\delta^{\star}}{\left\|\mathbf{R}_{M}\right\|_{2\rightarrow\infty}}.
    \label{eq:slack_cert}
\end{align}
\end{corollary}
\begin{proof}
The constraints of \eqref{eq:slack_lp} enforce $\delta^{\star}\leq\delta_{\infty}\left(\boldsymbol{\lambda}^{\star}\right)$. If $\delta^{\star}>0$, then $\boldsymbol{\lambda}^{\star}\in\intset\left(\Lambda\right)$ and $\M\left(\boldsymbol{\alpha}\right)\boldsymbol{\lambda}^{\star}=\w_{\mathrm{task}}$. Applying \Cref{thm:certificate} with $\boldsymbol{\lambda}_{0}=\boldsymbol{\lambda}^{\star}$ gives \eqref{eq:slack_cert}.
\end{proof}
Let $\T\subset\R^{6}$ be a compact set of task wrenches. The UAV is uniformly contact-persistently fully actuated over $\T$ with margin $\epsilon>0$ if $\B_{\epsilon}\left(\zero_{6}\right)\subseteq\W_{\mathrm{res}}\left(\boldsymbol{\alpha},\w\right)$ for all $\w\in\T$.
\begin{theorem}[Uniform Residual Authority]
\label{thm:uniform}
For compact $\T$, a positive uniform residual margin exists if and only if $\T\subset\intset\left(\W\left(\boldsymbol{\alpha}\right)\right)$. The largest uniform margin is
\begin{align}
\epsilon_{\T}^{\star}\left(\boldsymbol{\alpha}\right)=\dist\left(\T,\bd\W\left(\boldsymbol{\alpha}\right)\right).
    \label{eq:uniform_margin}
\end{align}
\end{theorem}
\begin{proof}
Pointwise existence follows from \Cref{thm:main}. If $\T\subset\intset\left(\W\left(\boldsymbol{\alpha}\right)\right)$ and $\T$ is compact, then $\T$ has positive distance from the closed set $\bd\W\left(\boldsymbol{\alpha}\right)$. This distance is precisely the largest uniform ball radius that can be placed around every $\w\in\T$ while remaining in $\W\left(\boldsymbol{\alpha}\right)$.
\end{proof}
The margin \eqref{eq:uniform_margin} has a direct closed-loop interpretation. Suppose a high-level controller requests
\begin{align}
    \w_{d}\left(t\right)=\w_{\mathrm{task}}\left(t\right)+\w_{\mathrm{stab}}\left(t\right),
    \label{eq:wrench_decomp}
\end{align}
where $\w_{\mathrm{stab}}$ is the stabilization and disturbance-rejection wrench. If $\w_{\mathrm{task}}\left(t\right)\in\T$ and $\left\|\w_{\mathrm{stab}}\left(t\right)\right\|<\epsilon_{\T}^{\star}\left(\boldsymbol{\alpha}\right)$ for all $t$, then $\w_{d}\left(t\right)\in\W\left(\boldsymbol{\alpha}\right)$ for all $t$. Thus, residual authority gives a controller-agnostic feasibility layer. Any pose or force controller whose stabilizing wrench remains inside the residual ball is implementable without violating actuator bounds.

\begin{corollary}[Robust Implementability and Margin-Aware Filtering]
\label{cor:robust_filter}
Assume that $\W\left(\boldsymbol{\alpha}\right)$ admits \eqref{eq:h_rep}. Let $\bar{w}_{\mathrm{stab}}\geq0$ and $\bar{w}_{\mathrm{unc}}\geq0$ be bounds on the stabilization wrench and the unmodeled interaction wrench, respectively. If
\begin{align}
\inf_{\w\in\T}\varphi_{\W}\left(\boldsymbol{\alpha},\w\right)
    \geq
    \bar{w}_{\mathrm{stab}}+\bar{w}_{\mathrm{unc}},
    \label{eq:robust_margin_condition}
\end{align}
then every commanded wrench of the form
\begin{align}
\w_{d}
    =
    \w_{\mathrm{task}}+\w_{\mathrm{stab}}+\w_{\mathrm{unc}};
    ~~
    \w_{\mathrm{task}}\in\T,
\end{align}
with $\left\|\w_{\mathrm{stab}}\right\|\leq\bar{w}_{\mathrm{stab}}$ and $\left\|\w_{\mathrm{unc}}\right\|\leq\bar{w}_{\mathrm{unc}}$ belongs to $\W\left(\boldsymbol{\alpha}\right)$. Furthermore, a desired task command $\w_{\mathrm{cmd}}$ can be filtered by the convex program
\begin{align}
\w_{\mathrm{cmd}}^{\star}
    \in
    \argmin_{\w\in\mathcal{C}}\;&\left\|\w-\w_{\mathrm{cmd}}\right\|^{2}
    \label{eq:margin_filter}\\
    \mathrm{s.t.}\;&
    \mathbf{a}_{j}^{\top}\w
    \leq
    b_{j}-\epsilon_{\mathrm{req}}\left\|\mathbf{a}_{j}\right\|,
    \quad j=1,\ldots,q,\nonumber
\end{align}
where $\epsilon_{\mathrm{req}}=\bar{w}_{\mathrm{stab}}+\bar{w}_{\mathrm{unc}}$ and $\mathcal{C}$ is any convex set encoding task-level contact-force limits. The filter projects the desired wrench onto the set of commands that preserve the required stabilizing and uncertainty margin.
\end{corollary}
\begin{proof}
Condition \eqref{eq:robust_margin_condition} implies that the ball of radius $\bar{w}_{\mathrm{stab}}+\bar{w}_{\mathrm{unc}}$ centered at each $\w\in\T$ lies in $\W\left(\boldsymbol{\alpha}\right)$ by \Cref{cor:signed_margin_ball}. The feasibility of all admissible $\w_{d}$ follows from the triangle inequality. The constraints in \eqref{eq:margin_filter} are affine half-space tightenings of \eqref{eq:h_rep}. Hence, the filtered command satisfies $\varphi_{\W}\left(\boldsymbol{\alpha},\w_{\mathrm{cmd}}^{\star}\right)\geq\epsilon_{\mathrm{req}}$.
\end{proof}

\section{Morphology-Aware Evaluation}
\label{sec:morphology}
Force and torque have different units, so numerical comparison is performed in normalized wrench coordinates. Let $F_{0}>0$ and $\ell_{0}>0$ be characteristic force and length scales, and define $\mathbf{D}_{w}=\blkdiag\left(\dfrac{1}{F_{0}}\mathbf{I}_{3},\dfrac{1}{\ell_{0}F_{0}}\mathbf{I}_{3}\right)$ and $\z=\mathbf{D}_{w}\w$. The normalized allocation map and wrench set are $\z=\A\left(\boldsymbol{\alpha}\right)\boldsymbol{\lambda}$ and $\A\left(\boldsymbol{\alpha}\right)=\mathbf{D}_{w}\M\left(\boldsymbol{\alpha}\right)$, where $\Z\left(\boldsymbol{\alpha}\right)=\left\{\z:\z=\A\left(\boldsymbol{\alpha}\right)\boldsymbol{\lambda},\;\boldsymbol{\lambda}\in\Lambda\right\}$. The normalized residual margin is evaluated using the signed margin of the normalized polytope. Thus, for feasible interior task wrenches it equals $\dist\left(\z_{\mathrm{task}},\bd\Z\left(\boldsymbol{\alpha}\right)\right)$, whereas infeasible exterior points receive negative margin. Here $\bd\Z\left(\boldsymbol{\alpha}\right)$ denotes the topological boundary of $\Z\left(\boldsymbol{\alpha}\right)$.

To represent a broad class of morphologies, consider $N$ thrust modules arranged on a circle of radius $\ell$. Let $\psi_{i}=\dfrac{2\pi\left(i-1\right)}{N}$ with $\mathbf{r}_{i}=\ell\begin{bmatrix}\cos\left(\psi_{i}\right)&\sin\left(\psi_{i}\right)&0\end{bmatrix}^{\top}$. Let the radial and tangential unit vectors be $\mathbf{e}_{r,i}=\begin{bmatrix}\cos\left(\psi_{i}\right)&\sin\left(\psi_{i}\right)&0\end{bmatrix}^{\top}$ and $\mathbf{e}_{t,i}=\begin{bmatrix}-\sin\left(\psi_{i}\right)&\cos\left(\psi_{i}\right)&0\end{bmatrix}^{\top}$, respectively. A generic tilted thrust direction is
\begin{align}
    \mathbf{b}_{i}\left(\gamma_{i},\eta_{i}\right)&=\cos\left(\gamma_{i}\right)\e_{3}
    +\sin\left(\gamma_{i}\right)\left(\cos\left(\eta_{i}\right)\mathbf{e}_{r,i}+\sin\left(\eta_{i}\right)\mathbf{e}_{t,i}\right).
    \label{eq:tilt_direction}
\end{align}
The coplanar case corresponds to $\gamma_{i}=0$. A fixed-tilt design corresponds to constant $\gamma_{i}$ and $\eta_{i}$. A variable-tilt design allows $\left(\gamma_{i},\eta_{i}\right)$ to be selected inside admissible limits.

\begin{proposition}[Coplanar Residual Deficiency]
\label{prop:coplanar}
If all thrust directions are parallel to $\e_{3}$, then $\rank\left(\M\right)\leq4$. Hence, no six-dimensional residual ball can be contained in the feasible wrench set.
\end{proposition}
\begin{proof}
All force columns lie in $\operatorname{span}\{\e_{3}\}$. The moment arms generate roll and pitch torques through $\mathbf{r}_{i}\times\e_{3}$, and drag contributes yaw torque. Therefore, the column span has at most one force dimension and three torque dimensions.
\end{proof}
For contact tasks, define the directional pushing set in normalized coordinates as
\begin{align}
    \T_{\mathrm{push}}=\left\{\mathbf{D}_{w}\left(\begin{bmatrix}mg\e_{3}\\\zero_{3}\end{bmatrix}+\begin{bmatrix}F\mathbf{n}\\\boldsymbol{\tau}_{c}\end{bmatrix}\right):F\in\left[0,F_{\max}\right],\;\left\|\boldsymbol{\tau}_{c}\right\|\leq\tau_{\max}\right\},
    \label{eq:t_push}
\end{align}
where $\mathbf{n}\in\mathbb{S}^{2}$ is the desired pushing direction. The signed task margin is
\begin{align}
    \varphi_{\mathrm{push}}^{\star}\left(\boldsymbol{\alpha}\right)
    :=
    \inf_{\z\in\T_{\mathrm{push}}}
    \varphi_{\Z}\left(\boldsymbol{\alpha},\z\right),
    \label{eq:signed_push_margin}
\end{align}
where $\varphi_{\Z}$ is the normalized analogue of \eqref{eq:signed_margin_w}. The contact-persistent margin is positive if and only if $\varphi_{\mathrm{push}}^{\star}\left(\boldsymbol{\alpha}\right)>0$, and the certified residual radius is $\epsilon_{\mathrm{push}}^{\star}\left(\boldsymbol{\alpha}\right)=\max\left\{0,\varphi_{\mathrm{push}}^{\star}\left(\boldsymbol{\alpha}\right)\right\}$. A morphology-selection problem is therefore
\begin{align}
\boldsymbol{\alpha}^{\star}\in\argmax_{\boldsymbol{\alpha}\in\Aset}\varphi_{\mathrm{push}}^{\star}\left(\boldsymbol{\alpha}\right),
    \label{eq:morph_opt}
\end{align}
subject to hover feasibility and actuator constraints. For a prescribed residual margin $\epsilon_{0}>0$, the maximum certified push force is
\begin{align}
    F_{\mathrm{push}}^{\star}=\sup\left\{F_{\max}:\varphi_{\mathrm{push}}^{\star}\left(\boldsymbol{\alpha};F_{\max},\tau_{\max}\right)\geq\epsilon_{0}\right\}.
    \label{eq:max_push}
\end{align}
This is different from maximum horizontal force since it explicitly reserves a six-dimensional residual stabilization margin.
\begin{figure}[t]
\centering
\begin{tikzpicture}[
    font=\footnotesize,
    >=stealth,
    axis/.style={->,thick,draw=cNavy},
    poly/.style={fill=cGray,draw=cTeal,thick},
    margin/.style={draw=cTeal,dashed,thick},
    callout/.style={->,thick,draw=cGold},
    task/.style={circle,fill=cGold,draw=cNavy,inner sep=1.25pt},
    labelbox/.style={fill=none,inner sep=1.5pt,text=cNavy}
]
   
    \draw[axis] (-2.45,0) -- (2.65,0)
        node[anchor=west] {$z_{1}$};
    \draw[axis] (0,-1.55) -- (0,1.60)
        node[anchor=south] {$z_{2}$};

    \coordinate (p1) at (-1.60,-0.82);
    \coordinate (p2) at ( 1.45,-0.58);
    \coordinate (p3) at ( 1.78, 0.82);
    \coordinate (p4) at (-1.22, 1.05);
    \filldraw[poly] (p1) -- (p2) -- (p3) -- (p4) -- cycle;

    \coordinate (ztask) at (0.32,0.18);
    \draw[margin] (ztask) circle (0.42);
    \node[task,label={[anchor=west,xshift=2pt,yshift=-6pt]right:{$\z_{\mathrm{task}}$}}] at (ztask) {};

    \draw[callout] (ztask) -- ++(0.42,0)
        node[midway,above,xshift=-2pt,yshift=-1pt] {$\epsilon^{\star}$};

    \node[labelbox,anchor=center] at (-0.88,0.62)
        {$\Z\!\left(\boldsymbol{\alpha}\right)$};

    \draw[callout] (-1.95,-1) -- (p1);
    \node[labelbox,anchor=north east,align=right] at (-1.95,-1)
        {$\bd{\Z}\!\left(\boldsymbol{\alpha}\right)$};

    \node[labelbox,anchor=south west,align=left,xshift=-10pt] at (0.80,0.48)
        {$\B_{\epsilon^{\star}}\!\left(\z_{\mathrm{task}}\right)$};
\end{tikzpicture}
\caption{Conceptual illustration of residual authority in a two-dimensional projection of the normalized wrench polytope. The task wrench is contact-persistently feasible only if a nonzero neighborhood of radius $\epsilon^{\star}$ remains inside $\Z\!\left(\boldsymbol{\alpha}\right)$.}
\label{fig:concept}
\end{figure}
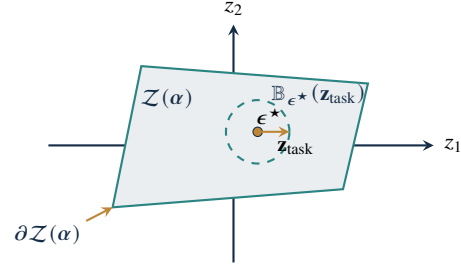

The residual authority margin can be computed in several equivalent ways depending on the representation of the feasible wrench set. If an $H$-representation of $\Z\left(\boldsymbol{\alpha}\right)$ is available, that is, if the polytope is written as the intersection of finitely many half spaces $\mathbf{a}_{j}^{\top}\z\leq b_{j}$, the signed normalized margin for a sample $\z_{0}$ is obtained directly as
\begin{align}
    \varphi_{\Z}\left(\boldsymbol{\alpha},\z_{0}\right)
    =
    \min_{j}
    \dfrac{b_{j}-\mathbf{a}_{j}^{\top}\z_{0}}
    {\left\|\mathbf{a}_{j}\right\|}.
    \label{eq:normalized_h_margin}
\end{align}
Positive values certify interiority, zero identifies boundary contact, and negative values identify infeasible task wrenches. For an overactuated architecture, one can avoid explicit polytope construction by using directional reachability. Given a unit direction $\mathbf{d}\in\R^{6}$, define
\begin{align}
    s^{\star}\left(\z_{0},\mathbf{d}\right)
    =
    \max_{s,\boldsymbol{\lambda}}\; s
    \label{eq:direction_lp}
\end{align}
subject to $\A\left(\boldsymbol{\alpha}\right)\boldsymbol{\lambda}=\z_{0}+s\mathbf{d}$ with $\boldsymbol{\lambda}\in\Lambda$ and $s\geq 0$. Then $s^{\star}\left(\z_{0},\mathbf{d}\right)$ is the distance one can move from $\z_{0}$ in the direction $\mathbf{d}$ before leaving the feasible set. The exact radius can be written as $\inf_{\left\|\mathbf{d}\right\|=1}s^{\star}\left(\z_{0},\mathbf{d}\right)$. If $\mathcal{D}$ is a finite set of unit directions, the sampled quantity
\begin{align}
    \hat{\epsilon}_{z}\left(\z_{0}\right)
    =
    \min_{\mathbf{d}\in\mathcal{D}}
    s^{\star}\left(\z_{0},\mathbf{d}\right)
    \label{eq:sampled_radius}
\end{align}
is a directional diagnostic rather than, by itself, a conservative certificate of the Euclidean ball, because the minimum is taken over fewer directions than the unit sphere. A certified finite-direction implementation requires either an $H$-representation, the actuator-slack certificate in \Cref{thm:certificate}, or an explicit direction-net covering error. For a sampled task set $\mathcal{S}_{\T}=\{\z_{1},\ldots,\z_{N_{s}}\}\subset\T_{\mathrm{push}}$, the sampled directional task margin is
\begin{align}
    \hat{\epsilon}_{\mathrm{push}}
    =
    \min_{k\in\{1,\ldots,N_{s}\}}
    \hat{\epsilon}_{z}\left(\z_{k}\right).
\end{align}
This computation is useful as a design diagnostic when $N>6$, when $\M\left(\boldsymbol{\alpha}\right)$ has nontrivial null-space redundancy, or when the actuator bounds include additional power, rate, or health constraints. In the square full-rank case, the exact closed-form expression used in \Cref{sec:numerics} is recovered from the inverse allocation map.

The same metric can distinguish a fixed morphology from a task-adaptive morphology. Let $\varphi\left(\boldsymbol{\alpha},\z\right):=\varphi_{\Z}\left(\boldsymbol{\alpha},\z\right)$ denote the signed normalized margin. For a task set $\T_{z}$, the best fixed morphology margin is
\begin{align}
    \epsilon_{\mathrm{fix}}^{\star}
    =
    \sup_{\boldsymbol{\alpha}\in\Aset}
    \inf_{\z\in\T_{z}}
    \varphi\left(\boldsymbol{\alpha},\z\right),
    \label{eq:fixed_margin}
\end{align}
whereas the best task-adaptive quasi-static margin is
\begin{align}
    \epsilon_{\mathrm{ad}}^{\star}
    =
    \inf_{\z\in\T_{z}}
    \sup_{\boldsymbol{\alpha}\in\Aset}
    \varphi\left(\boldsymbol{\alpha},\z\right).
    \label{eq:adaptive_margin}
\end{align}

\begin{proposition}[Non-reduction of Pointwise Authority Under Task-Adaptive Morphology]
\label{prop:nonreduction}
For any compact task set $\T_{z}$ and admissible morphology set $\Aset$, $\epsilon_{\mathrm{ad}}^{\star}
    \geq
    \epsilon_{\mathrm{fix}}^{\star}$.
\end{proposition}
\begin{proof}
The claim is the standard minimax inequality
\begin{align}
\sup_{\boldsymbol{\alpha}\in\Aset}\inf_{\z\in\T_{z}}\varphi\left(\boldsymbol{\alpha},\z\right)
    \leq
    \inf_{\z\in\T_{z}}\sup_{\boldsymbol{\alpha}\in\Aset}\varphi\left(\boldsymbol{\alpha},\z\right).
\end{align}
\end{proof}
\Cref{prop:nonreduction} should be interpreted carefully. It is a quasi-static actuation statement, not a claim about closed-loop performance with tilt servos. A dynamic variable-tilt implementation must also preserve residual authority during transients, satisfy tilt-rate limits, and avoid singular configurations. Nevertheless, the result identifies the attractiveness of morphology adaptation for contact-rich UAVs as it can align the feasible wrench polytope with the current task direction while still reserving stabilizing authority.

\begin{table*}[t]
\centering
\caption{Distinguishing morphology classes based on the residual-authority framework.}
\label{tab:morphology_classes}
\begin{tabular}{p{0.23\textwidth}p{0.22\textwidth}p{0.25\textwidth}p{0.22\textwidth}}
\toprule
Morphology class & Structural property & Residual-authority behavior & Design implication\\
\midrule
Coplanar multirotor & $\rank\left(\M\right)\leq4$ for parallel thrust axes & No six-dimensional residual wrench ball at fixed attitude & Contact requires body reorientation or task relaxation\\
Fixed-tilt fully actuated multirotor & Can satisfy $\rank\left(\M\right)=6$ & Positive margin only when the task set lies strictly inside $\W$ & Tilt angle must balance lateral authority and hover efficiency\\
Overactuated fixed morphology & $N>6$ with null-space redundancy & Allocation can trade tracking, power, and residual margin & Redundancy should be used to maximize distance from saturation\\
Variable-tilt morphology & $\M\left(\boldsymbol{\alpha}\right)$ changes with morphology & Task-adaptive margins can improve directional contact capacity & Requires dynamic allocation with servo constraints\\
Task-biased contact morphology & Full rank or overactuated near a preferred direction & High best-direction authority but possibly anisotropic & Useful for persistent pushing, bracing, or inspection along known normals\\
\bottomrule
\end{tabular}
\end{table*}

\section{Numerical Evaluation}\label{sec:numerics}
The numerical study evaluates the signed residual-authority framework on an abstract symmetric hexarotor. The objective is to isolate the distinction among structural rank, task feasibility, residual stabilization authority, and margin-aware command modification. The study is not intended as a validation of a particular airframe. The parameters represent a small aerial physical-interaction platform. The mass follows the scale of the coaxial-tiltrotor pushing experiments in \cite{hwang2024pushing}, whereas the arm length, drag-to-thrust coefficient, and rotor limit are representative design-scale values. All contact forces and torques are normalized by $mg$ and $\ell mg$, respectively, so the reported results should be interpreted as morphology and authority trends rather than hardware-specific performance predictions.

The parameters are $N=6$, $m=3.3~\mathrm{kg}$, $\ell=0.35~\mathrm{m}$, $c_{\tau}=0.025~\mathrm{m}$, and $0\leq\lambda_{i}\leq13~\mathrm{N}$ for $i=1,\ldots,6$. The wrench normalization uses $F_{0}=mg$ and $\ell_{0}=\ell$. Unless stated otherwise, the pushing direction is $\mathbf{n}=\e_{1}$, the normalized force satisfies $\rho=F/(mg)\in[0,0.15]$, and the normalized contact-torque uncertainty satisfies $\left\|\boldsymbol{\mu}\right\|\leq0.04$, where $\boldsymbol{\mu}=\boldsymbol{\tau}_{c}/(\ell mg)$. The required residual margin is $\epsilon_{0}=0.025$.

For fixed-tilt cases, the alternating mixed-tilt pattern is $\gamma_{i}=\gamma,
~
\eta_{i}=\left(-1\right)^{i+1}\dfrac{\pi}{4};
~ i=1,\ldots,6$. This pattern provides a representative full-rank geometry for evaluating the proposed quantities. Whenever $\A\left(\boldsymbol{\alpha}\right)$ is nonsingular, define $\mathbf{B}=\A^{-1}$ and denote the $i$th row of $\mathbf{B}$ by $\boldsymbol{\beta}_{i}^{\top}$. For a normalized wrench $\z_{0}$, the unique allocation is $\boldsymbol{\lambda}_{0}=\mathbf{B}\z_{0}$. Its exact pointwise signed margin is
\begin{align}
\varphi_{\Z}\left(\boldsymbol{\alpha},\z_{0}\right)
=&
\min_{i}
\left\{
\dfrac{\lambda_{i,\max}-\lambda_{0,i}}
{\left\|\boldsymbol{\beta}_{i}\right\|},
\dfrac{\lambda_{0,i}-\lambda_{i,\min}}
{\left\|\boldsymbol{\beta}_{i}\right\|}
\right\}.
\label{eq:square_signed_margin}
\end{align}
Thus, positive, zero, and negative values indicate strict interiority, boundary contact, and infeasibility, respectively.

For the continuous pushing set, let
\begin{align}
\z
=&
\z_{h}
+\rho
\begin{bmatrix}
\mathbf{n}^{\top}&\zero_{3}^{\top}
\end{bmatrix}^{\top}
+
\begin{bmatrix}
\zero_{3}^{\top}&\boldsymbol{\mu}^{\top}
\end{bmatrix}^{\top},
&
\z_{h}
=&
\begin{bmatrix}
\e_{3}^{\top}&\zero_{3}^{\top}
\end{bmatrix}^{\top}.
\label{eq:numerical_task_wrench}
\end{align}
Partitioning $\boldsymbol{\beta}_{i}=\begin{bmatrix}\boldsymbol{\beta}_{i,f}^{\top}&\boldsymbol{\beta}_{i,\tau}^{\top}\end{bmatrix}^{\top}$, the exact worst-case lower and upper actuator slacks over $\rho\in[0,\rho_{\max}]$ and $\left\|\boldsymbol{\mu}\right\|\leq\mu_{\max}$ are
\begin{align}
s_{i}^{-}
=&
\boldsymbol{\beta}_{i}^{\top}\z_{h}
+\min\left\{0,\rho_{\max}\boldsymbol{\beta}_{i,f}^{\top}\mathbf{n}\right\}
-\mu_{\max}\left\|\boldsymbol{\beta}_{i,\tau}\right\|
-\lambda_{i,\min},
\label{eq:exact_lower_slack}\\
s_{i}^{+}
=&
\lambda_{i,\max}
-\boldsymbol{\beta}_{i}^{\top}\z_{h}
-\max\left\{0,\rho_{\max}\boldsymbol{\beta}_{i,f}^{\top}\mathbf{n}\right\}
-\mu_{\max}\left\|\boldsymbol{\beta}_{i,\tau}\right\|.
\label{eq:exact_upper_slack}
\end{align}
The exact signed task-set margin and the corresponding nonnegative residual radius are therefore
\begin{align}
\varphi_{\mathrm{push}}^{\star}
=&
\min_{i}
\dfrac{\min\left\{s_{i}^{-},s_{i}^{+}\right\}}
{\left\|\boldsymbol{\beta}_{i}\right\|},
&
\epsilon_{\mathrm{push}}^{\star}
=&
\max\left\{0,\varphi_{\mathrm{push}}^{\star}\right\}.
\label{eq:exact_signed_task_margin}
\end{align}
All task-set values below use the continuous torque ball analytically; no finite sampling of contact-torque directions is used.

\begin{table}[t]
\centering
\caption{Signed residual authority for $\rho_{\max}=0.15$ and $\mu_{\max}=0.04$. The certified force $F_{\mathrm{push}}^{\star}$ preserves $\epsilon_{0}=0.025$.}
\label{tab:results}
\begin{tabular}{ccccc}
\toprule
$\gamma$ (deg) & $\rank\left(\A\right)$ & Feasible? &
$\varphi_{\mathrm{push}}^{\star}$ & $F_{\mathrm{push}}^{\star}/mg$\\
\midrule
$0$  & $4$ & No  & --        & $0$\\
$10$ & $6$ & No  & $-0.0782$ & $0.0238$\\
$20$ & $6$ & No  & $-0.0231$ & $0.0873$\\
$30$ & $6$ & Yes & $0.0295$  & $0.1564$\\
$36$ & $6$ & Yes & $0.0500$  & $0.1879$\\
$37$ & $6$ & Yes & $0.0499$  & $0.1883$\\
$40$ & $6$ & Yes & $0.0474$  & $0.1860$\\
$50$ & $6$ & Yes & $0.0164$  & $0.1335$\\
\bottomrule
\end{tabular}
\end{table}

\begin{figure*}[t]
\centering
\begin{subfigure}[t]{0.485\textwidth}
\centering
\includegraphics[width=\linewidth]{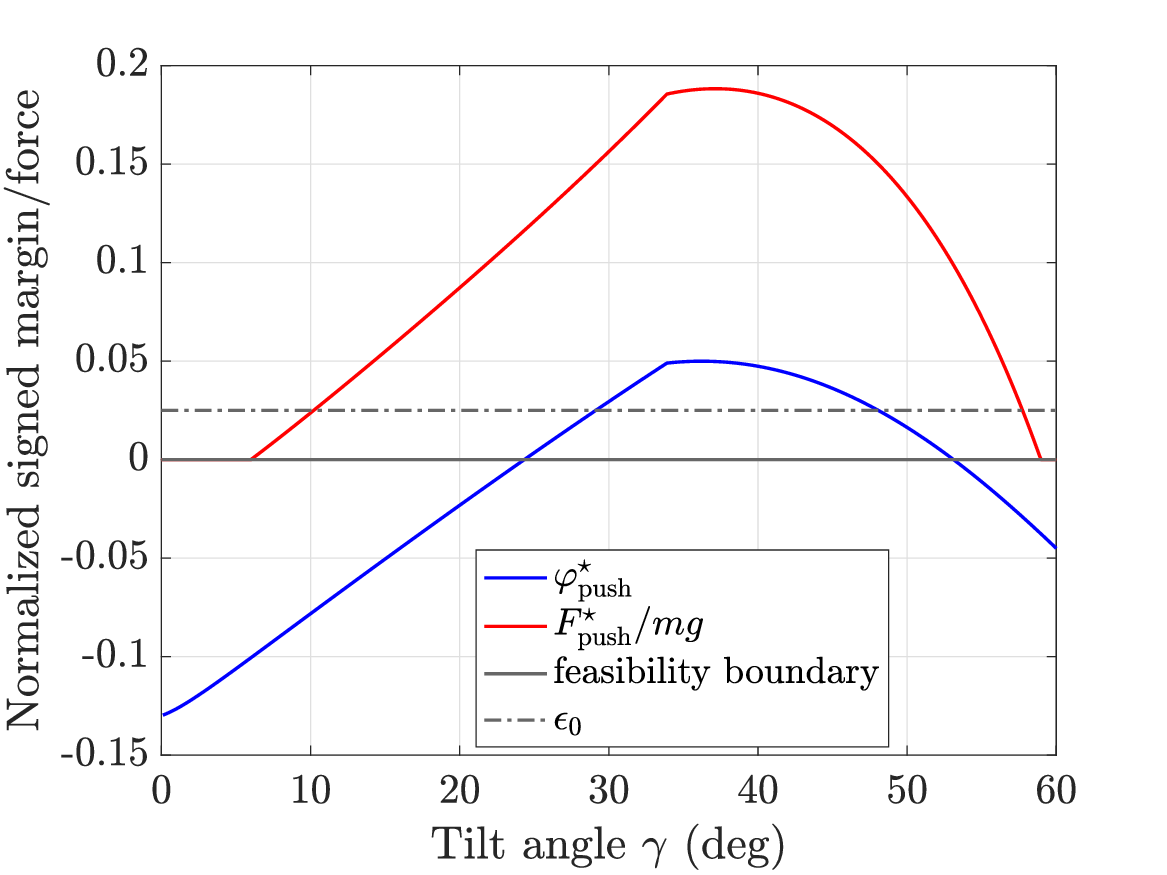}
\caption{Signed task margin and certified pushing force.}
\label{fig:margin_vs_tilt}
\end{subfigure}
\hfill
\begin{subfigure}[t]{0.485\textwidth}
\centering
\includegraphics[width=\linewidth]{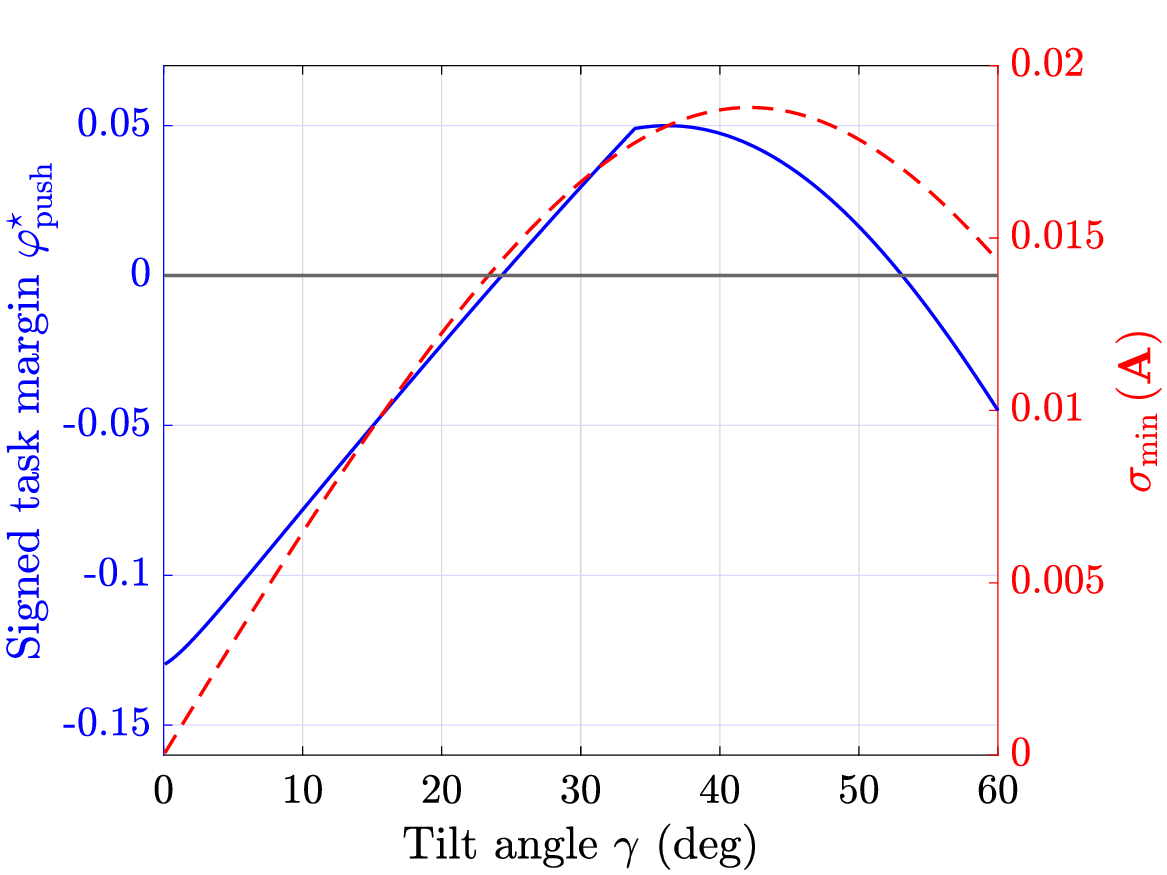}
\caption{Smallest singular value and signed task margin.}
\label{fig:conditioning_vs_margin}
\end{subfigure}
\caption{Tilt-angle study. The signed margin detects task infeasibility even when the allocation map remains full rank and well conditioned.}
\label{fig:tilt_study}
\end{figure*}

\subsection{Morphology Sweep and Directional Adaptation}

\Cref{tab:results} and \Cref{fig:margin_vs_tilt} show that the coplanar morphology is rank deficient, while the low-tilt cases $\gamma=10^{\circ}$ and $\gamma=20^{\circ}$ are rank six but have negative signed task margins. The prescribed task becomes feasible near $\gamma=24.3^{\circ}$ and loses feasibility again near $\gamma=53.1^{\circ}$. Requiring the stronger margin $\varphi_{\mathrm{push}}^{\star}\geq\epsilon_{0}$ restricts the useful interval to approximately $29.1^{\circ}\leq\gamma\leq48.0^{\circ}$. The maximum signed task margin is approximately $0.05$ at $\gamma=36.2^{\circ}$, whereas the certified pushing force peaks at approximately $0.1883mg$ near $\gamma=37.1^{\circ}$. These different optima reflect the distinction between maximizing the margin for a prescribed task set and maximizing admissible task magnitude subject to a prescribed residual margin.

The comparison in \Cref{fig:conditioning_vs_margin} further shows that numerical conditioning alone is not a task-feasibility certificate. The smallest singular value remains positive for every nonzero tilt in the sweep and reaches its maximum near $\gamma=42.2^{\circ}$, yet the signed task margin is negative for both small and excessive tilts. Hence, rank and singular-value information describe structural wrench generation, whereas $\varphi_{\mathrm{push}}^{\star}$ captures task loading, actuator bounds, hover demand, and contact uncertainty simultaneously.

Directional effects are evaluated using $\mathbf{n}\left(\theta\right)=\begin{bmatrix}\cos\left(\theta\right)&\sin\left(\theta\right)&0\end{bmatrix}^{\top}$ for $\theta\in[0,2\pi]$. For the fixed morphology $\gamma=37^{\circ}$,
\begin{align}
0.1667
\leq&
\dfrac{F_{\mathrm{push}}^{\star}\left(\theta\right)}{mg}
\leq
0.2003,
\end{align}
with mean $0.1820$. Selecting $\gamma\in[0^{\circ},60^{\circ}]$ pointwise for each direction changes the range to
\begin{align}
0.1667
\leq&
\dfrac{F_{\mathrm{push}}^{\star}\left(\theta\right)}{mg}
\leq
0.2630,
\end{align}
with mean $0.1978$.

\begin{figure}[t]
\centering

\includegraphics[width=\columnwidth]{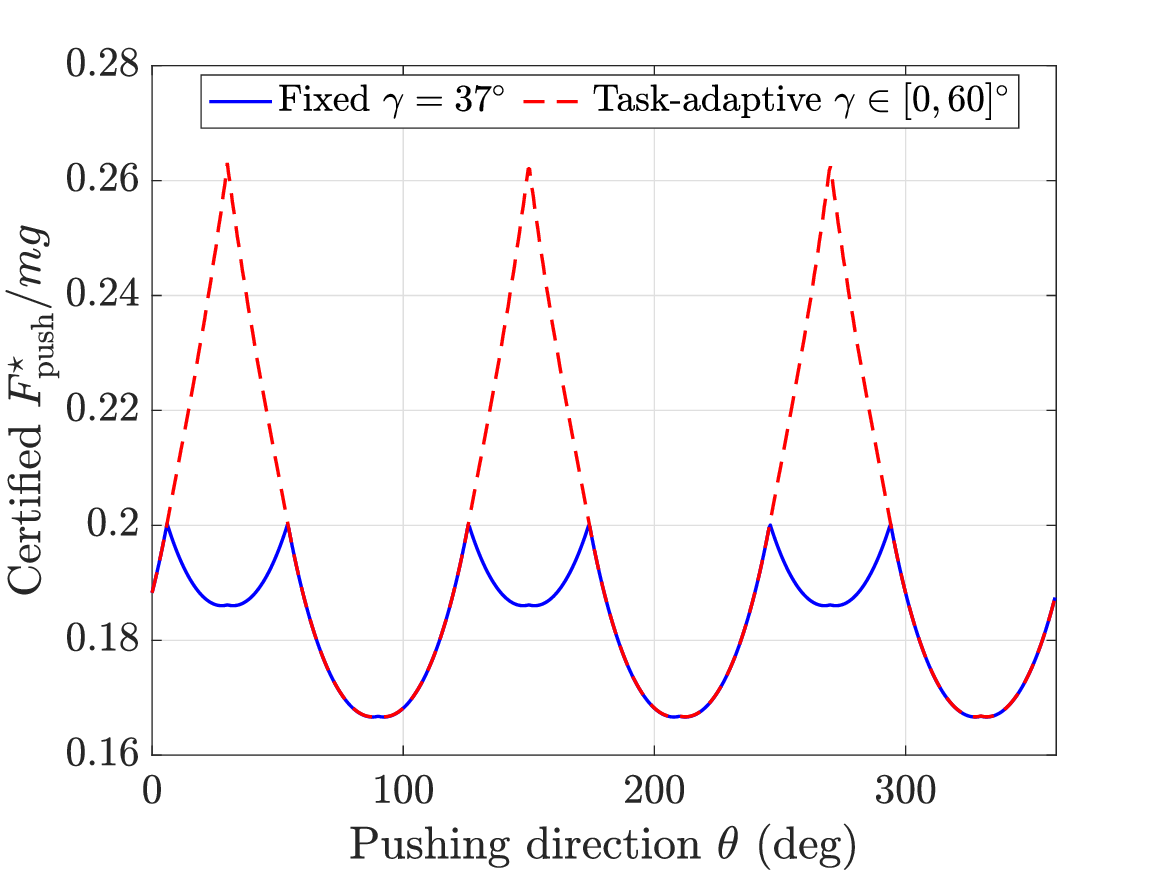}
\caption{Certified directional pushing capacity for the fixed morphology $\gamma=37^{\circ}$ and a task-adaptive morphology with $\gamma\in[0^{\circ},60^{\circ}]$.}
\label{fig:Fpush}
\end{figure}

As shown in \Cref{fig:Fpush}, task-adaptive tilt selection increases the mean certified capacity by approximately $8.7\%$ and the best-direction capacity by approximately $31.3\%$, while preserving the same required residual margin. The minimum directional capacities are nearly unchanged, indicating that adaptation primarily enlarges authority along favorable task directions rather than uniformly expanding the worst-direction capability.

\subsection{Signed-Margin and Robust-Loading Validation}

The signed-margin sweep varies the requested force limit over $\rho_{\max}\in[0,0.30]$ for $\gamma\in\{10^{\circ},20^{\circ},30^{\circ},36^{\circ},50^{\circ}\}$. Unlike the nonnegative residual radius, the signed quantity continues below zero and therefore shows the severity of infeasibility. For example, at $\gamma=36^{\circ}$, the task-set margin reaches $\epsilon_{0}$ at $\rho_{\max}=0.1879$ and reaches the feasibility boundary at $\rho_{\max}=0.2259$. For $\gamma=10^{\circ}$, these thresholds reduce to $0.0238$ and $0.0543$, respectively. The certified limits reported in \Cref{tab:results} are therefore precisely the intersections with $\varphi_{\mathrm{push}}^{\star}=\epsilon_{0}$, rather than maximum-force values obtained without reserving stabilization authority.

\begin{figure*}[t]
\centering
\begin{subfigure}[t]{0.485\textwidth}
\centering
\includegraphics[width=\linewidth]{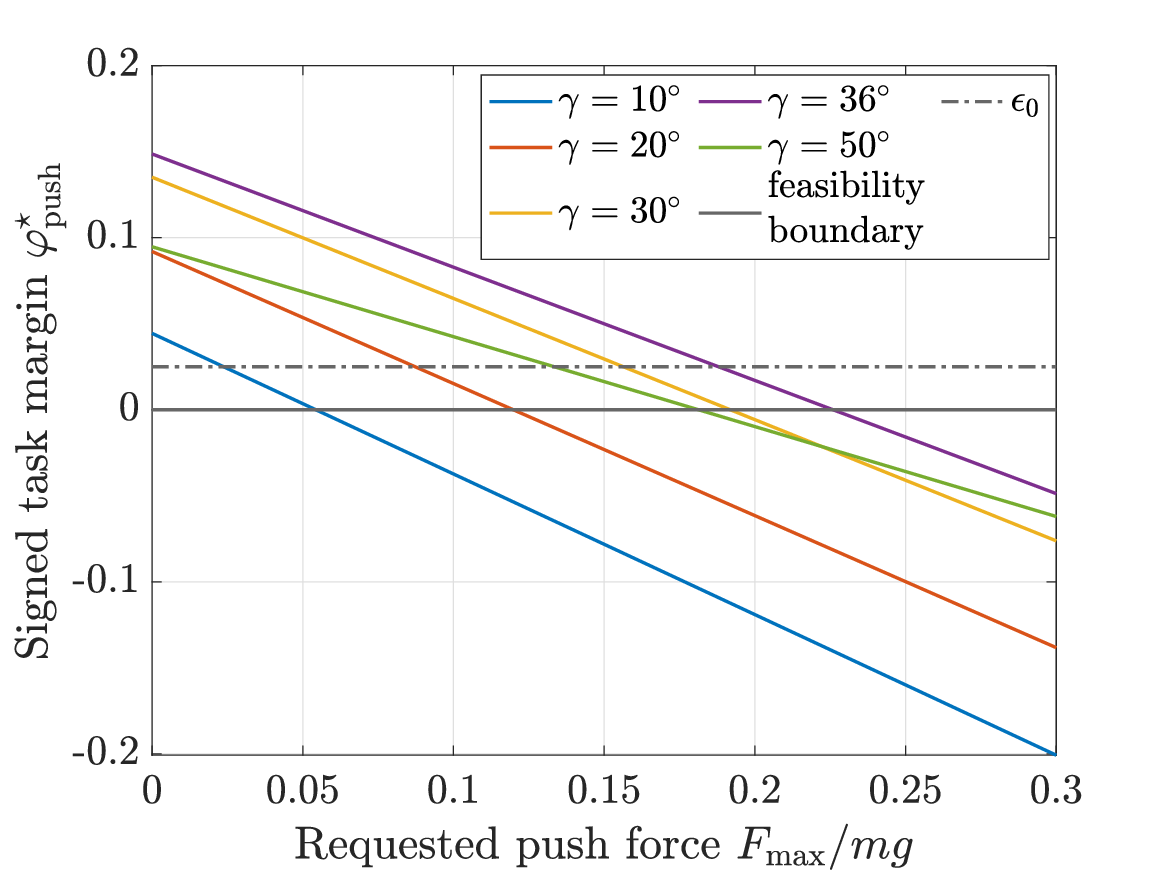}
\caption{Signed task margin versus requested pushing force.}
\label{fig:signed_margin_vs_push}
\end{subfigure}
\hfill
\begin{subfigure}[t]{0.485\textwidth}
\centering

\includegraphics[width=\linewidth]{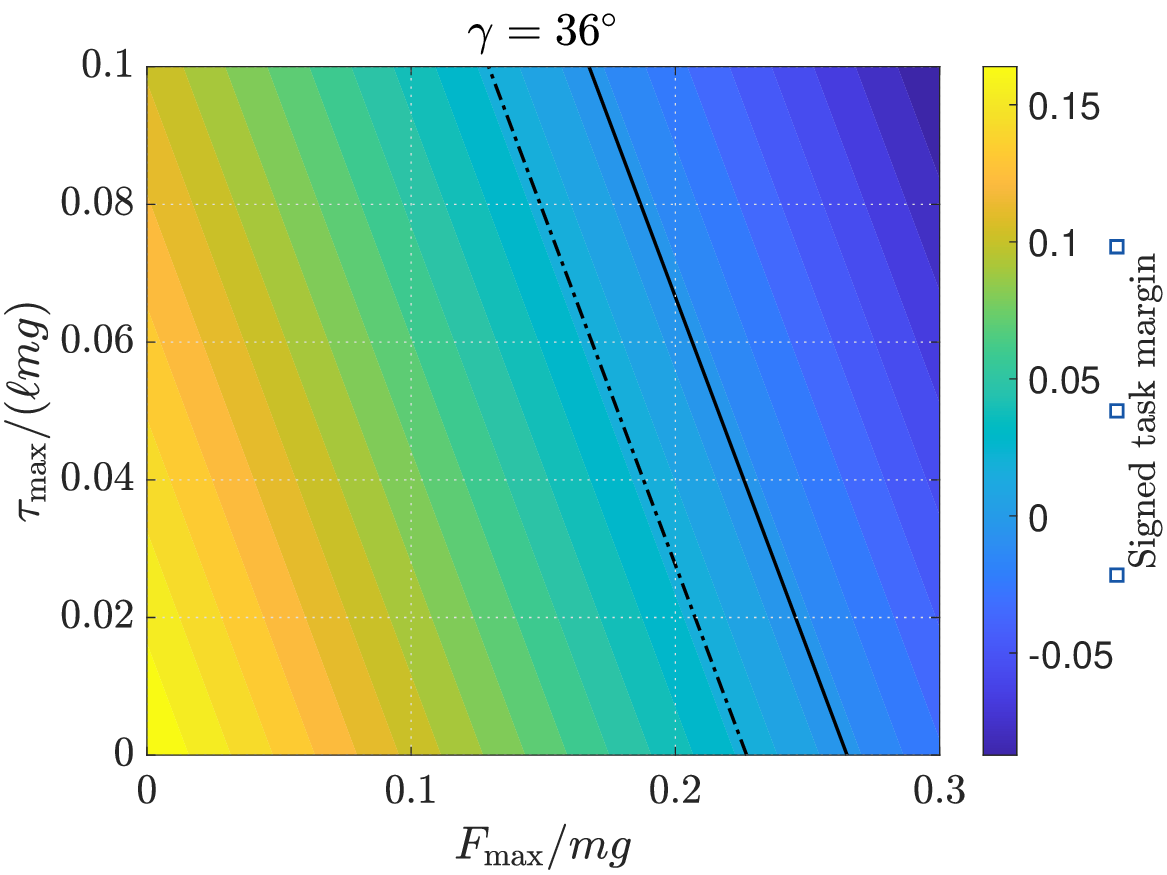}
\caption{Robust force--torque feasibility map at $\gamma=36^{\circ}$.}
\label{fig:robust_feasibility_map}
\end{subfigure}
\caption{Signed-margin validation under increasing contact loading. The solid zero contour identifies feasibility loss, while the $\epsilon_{0}$ contour identifies the admissible set with reserved stabilization authority.}
\label{fig:robust_loading}
\end{figure*}

The two-parameter map in \Cref{fig:robust_feasibility_map} exposes the force--torque tradeoff. At the nominal loading pair $\left(\rho_{\max},\mu_{\max}\right)=\left(0.15,0.04\right)$, the signed task margin is $0.05$. For $\mu_{\max}=0.04$, the $\epsilon_{0}$ and zero-margin force boundaries occur at $\rho_{\max}=0.1879$ and $\rho_{\max}=0.2259$, respectively. Conversely, at $\rho_{\max}=0.15$, the required-margin boundary occurs near $\mu_{\max}=0.0789$. This map gives a design chart for allocating residual authority among contact force, contact-torque uncertainty, and stabilization demand.

The robust implementability condition in \Cref{cor:robust_filter} is also evaluated with $\bar{w}_{\mathrm{stab}}=0.020$ and $\bar{w}_{\mathrm{unc}}=0.015$. Since
\begin{align}
\bar{w}_{\mathrm{stab}}+\bar{w}_{\mathrm{unc}}
=&
0.035
<
\varphi_{\mathrm{push}}^{\star}
=
0.05,
\end{align}
the complete task set remains feasible under all perturbations within the prescribed stabilization and uncertainty balls. A $20{,}000$-sample Monte Carlo check over the continuous task set and both perturbation balls produced no actuator-feasibility violations. This randomized check is used only as a numerical sanity test; the guarantee follows from the deterministic signed-margin condition.

\subsection{Allocation Certificate and Margin-Aware Filtering}

For the representative point $\gamma=36^{\circ}$, $\rho=0.15$, and $\boldsymbol{\mu}=\zero_{3}$, the exact allocation is
\begin{align}
\boldsymbol{\lambda}_{0}
=&
\begin{bmatrix}
6.515&10.098&3.240&6.823&10.253&3.086
\end{bmatrix}^{\top}\mathrm{N},
\label{eq:representative_allocation}
\end{align}
with normalized usages
\begin{align}
\dfrac{\boldsymbol{\lambda}_{0}}{\lambda_{\max}}
=&
\begin{bmatrix}
0.501&0.777&0.249&0.525&0.789&0.237
\end{bmatrix}^{\top}.
\label{eq:representative_usage}
\end{align}
The minimum actuator slack is $2.747~\mathrm{N}$. For this symmetric square architecture, the allocation is unique and all rows of $\A^{-1}$ have the same Euclidean norm. Consequently, the slack certificate equals the exact pointwise margin over the plotted force range; at $\rho=0.15$, both give approximately $0.0757$. The equality is architecture-specific. In a general overactuated system, \Cref{cor:slack_max} provides a lower bound whose tightness depends on the selected right inverse and the slack-maximizing allocation.

\begin{figure*}[t]
\centering
\begin{subfigure}[t]{0.485\textwidth}
\centering
\includegraphics[width=\linewidth]{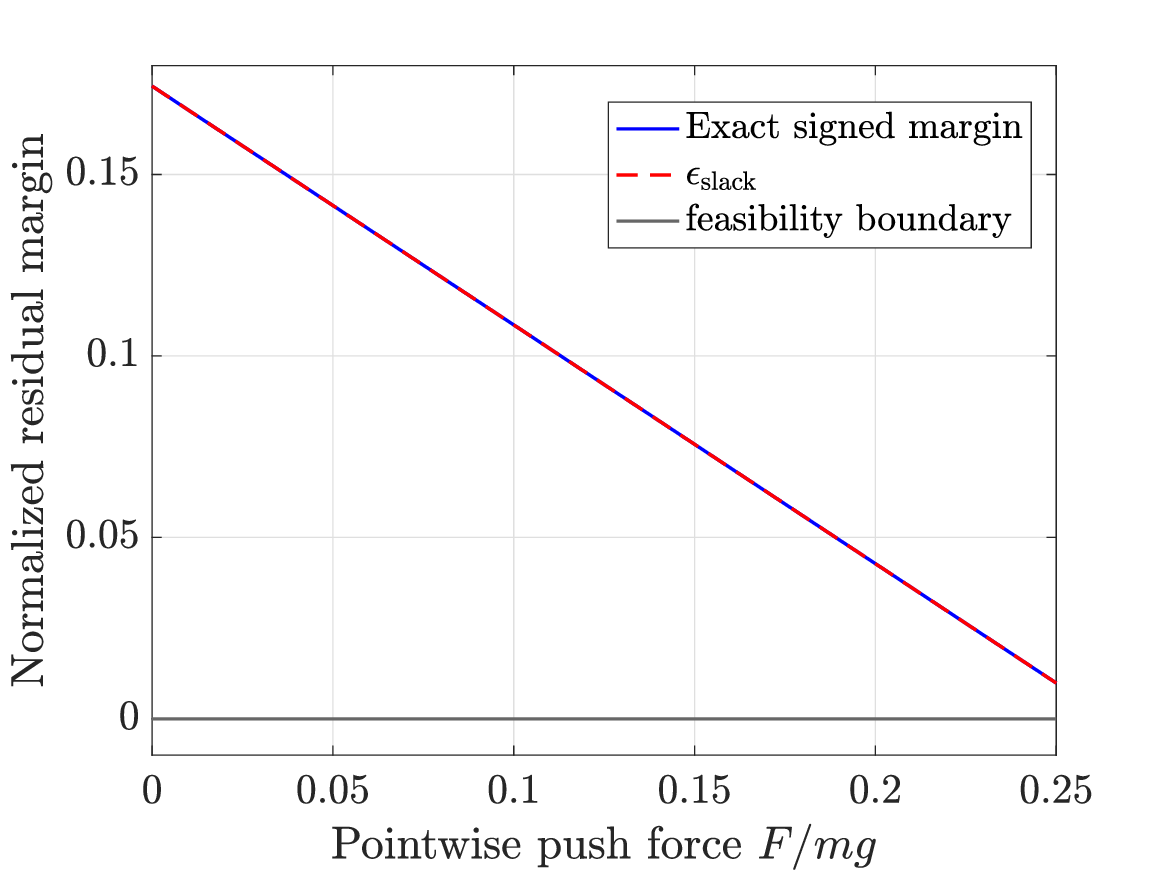}
\caption{Exact pointwise margin and actuator-slack certificate.}
\label{fig:certificate_comparison}
\end{subfigure}
\hfill
\begin{subfigure}[t]{0.485\textwidth}
\centering
\includegraphics[width=\linewidth]{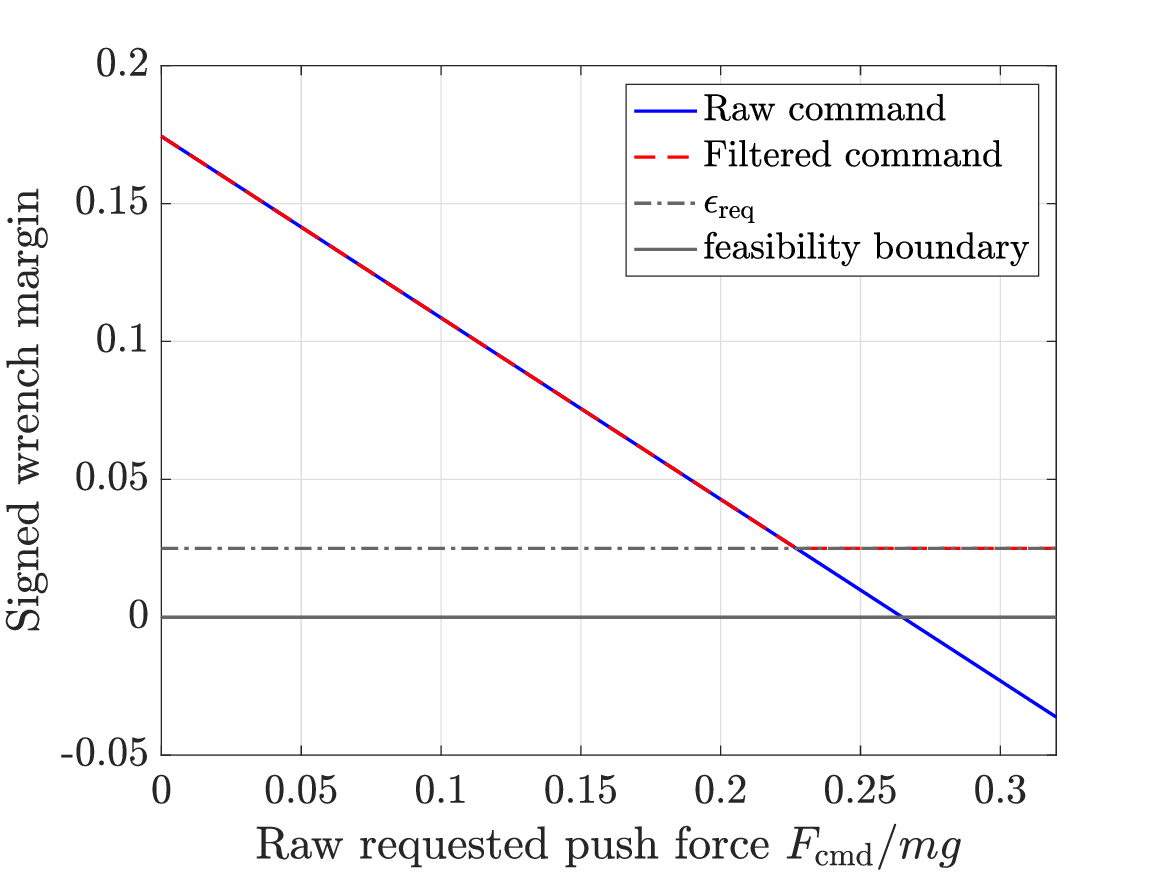}
\caption{Raw and margin-filtered command margins.}
\label{fig:margin_filter}
\end{subfigure}
\caption{Validation of the computable allocation certificate and the margin-aware wrench filter.}
\label{fig:certificate_filter}
\end{figure*}

The margin-aware filter is evaluated at $\gamma=36^{\circ}$ using raw commands $\rho_{\mathrm{cmd}}\in[0,0.32]$ and $\epsilon_{\mathrm{req}}=\epsilon_{0}$. The unfiltered command reaches the required-margin boundary at approximately $\rho_{\mathrm{cmd}}=0.2270$ and becomes infeasible near $\rho_{\mathrm{cmd}}=0.2650$. The filter is inactive below the required-margin boundary. Above it, the command is projected onto the tightened wrench polytope and the filtered signed margin remains exactly at or above $0.025$. Thus, the filter rejects only the portion of the requested wrench that would consume the reserved stabilization authority, rather than waiting until actuator infeasibility occurs.

\section{Discussion and Future Control Problems}
\label{sec:future}
The results also expose several open control-theoretic problems that merit further investigation. The certificates above show how residual authority can be embedded in real-time allocation and filtering, but several dynamic issues remain. First, the static margin should be combined with closed-loop pose and force dynamics to derive stability conditions under saturated allocation. Second, the morphology variable can be treated as a slow control input, producing a co-design problem in which $\boldsymbol{\alpha}$ is selected to maximize residual authority while the fast thrust allocation tracks the desired wrench. Third, for variable-tilt platforms, transient feasibility must account for servo-rate limits and possible loss of rank during reconfiguration. Finally, margin-aware safety filters for aerial physical interaction should enforce not only contact-force feasibility, but also preservation of the residual wrench authority required for stabilization.

The framework also clarifies a design tradeoff. Fixed tilts improve lateral force capability but may reduce vertical thrust efficiency. Variable tilt improves task adaptation but introduces servo dynamics, rate constraints, and possible transient authority loss. These effects are not captured by the static polytope alone. However, the residual authority margin provides the baseline constraint that any dynamic morphology controller must preserve.

\section{Conclusions}
This paper introduced contact-persistent full actuation as a residual wrench authority property for fully actuated UAVs in contact-rich tasks. The proposed framework replaces the purely rank-based view of full actuation with a constraint-aware condition based on the feasible wrench polytope. The main result shows that a task wrench admits nonzero six-dimensional residual authority if and only if it lies in the interior of the feasible wrench set, and that the maximum residual authority radius is the distance to the polytope boundary. A computable allocation-margin certificate, a signed-margin robust feasibility condition, and a task-set uniform version were derived. Numerical evaluation on an abstract tilted hexarotor showed that rank-six allocation can still be insufficient for pushing tasks, and that intermediate or adaptive tilt choices can preserve residual stabilization authority. Future work will integrate these margins into saturation-aware allocation, dynamic tilt control, and safety-certified aerial physical interaction.

\bibliographystyle{IEEEtran}
\bibliography{references}
\end{document}